\documentclass{article}

\usepackage{microtype}
\usepackage{graphicx}
\usepackage{caption}
\usepackage{subcaption}
\usepackage{booktabs} 

\usepackage{bm}
\usepackage{bbm}
\usepackage{amsmath}
\usepackage{amssymb}
\usepackage{amsfonts}
\usepackage{amsthm}

\usepackage{multirow}
\newtheorem{theorem}{Theorem}

\usepackage{hyperref}

\usepackage[accepted]{shmicml2021}

\icmltitlerunning{DP-InstaHide:  Provably Defusing Poisoning and Backdoor Attacks with Differentially Private Data Augmentations}

\begin{document}

\twocolumn[

\icmltitle{DP-InstaHide:  Provably Defusing Poisoning and Backdoor Attacks with Differentially Private Data Augmentations }



\icmlsetsymbol{equal}{*}

\begin{icmlauthorlist}
\icmlauthor{Eitan Borgnia}{md}
\icmlauthor{Jonas Geiping}{equal,sieg}
\icmlauthor{Valeriia Cherepanova}{equal,md}
\icmlauthor{Liam Fowl}{equal,md}
\icmlauthor{Arjun Gupta}{equal,md}
\icmlauthor{Amin Ghiasi}{equal,md}
\icmlauthor{Furong Huang}{md}
\icmlauthor{Micah Goldblum}{md}
\icmlauthor{Tom Goldstein}{md}
\end{icmlauthorlist}

\icmlaffiliation{md}{Dept. of Computer Science, University of Maryland, College Park, USA}
\icmlaffiliation{sieg}{Dept. of Electrical Engineering and Computer Science, University of Siegen, Siegen, Germany}
\icmlcorrespondingauthor{Amin Ghiasi}{amin@cs.umd.edu}

\icmlkeywords{Data Poisoning,Data Augmentation,Backdoor Attacks,Clean-label Poisoning, Targeted Data Poisoning,Defense,Differential Privacy,Mixup}

\vskip 0.3in
]



\printAffiliationsAndNotice{\icmlEqualContribution} 

\begin{abstract}
Data poisoning and backdoor attacks manipulate training data to induce security breaches in a victim model.
These attacks can be provably deflected using differentially private (DP) training methods, although this comes with a sharp decrease in model performance.   The InstaHide method has recently been proposed as an alternative to DP training that leverages supposed privacy properties of the mixup augmentation, although without rigorous guarantees.

In this work, we show that strong data augmentations, such as mixup and random additive noise, nullify
poison attacks while enduring only a small accuracy trade-off. 
To explain these finding, we propose a training method, DP-InstaHide, which combines the mixup regularizer with additive noise.  A rigorous analysis of DP-InstaHide shows that mixup does indeed have privacy advantages, and that training with $k$-way mixup provably yields at least $k$ times stronger DP guarantees than a naive DP mechanism.  Because mixup (as opposed to noise) is beneficial to model performance, DP-InstaHide provides a mechanism for achieving stronger empirical performance against poisoning attacks than other known DP methods.  

\end{abstract}

\section{Introduction}
\label{sec:intro}

As the capabilities of machine learning systems expand, so do their training data demands.  To satisfy this massive data requirement, developers create automated web scrapers that download data without human supervision.  The lack of human control over the machine learning pipeline may expose systems to \emph{poisoned} training data that induces pathologies in models trained on it.  Data poisoning and backdoor attacks may degrade accuracy or elicit incorrect predictions in the presence of a triggering visual feature \citep{shafahi2018poison, chen2017targeted}.

To combat this threat model, a number of defenses against data poisoning have emerged.  Certified defenses based on \emph{differential privacy} (DP) provably desensitize models to small changes in their training data by adding noise to either the data or the gradients used by their optimizer \citep{ma2019data}. 
When a model is trained using sufficiently strong DP, it is not possible to infer whether a small collection of data points were present in the training set by observing model behaviors, and it is therefore not possible to significantly alter model behaviors by introducing a small number of poisoned samples.
 
In this work, we show that strong data augmentations, specifically mixup \cite{zhang2017mixup} and its variants, provide state-of-the-art empirical defense against data poisoning, backdoor attacks, and even adaptive attacks.  This good performance can be explained by the differential privacy benefits of mixup.  Mixup augmentation is the basis of the InstaHide algorithm \citep{huang_instahide_2020}, which aims to create dataset privacy by averaging random image pairs and then multiplying the results by a random mask to inject randomness.  Unfortunately, the privacy claims in the original InstaHide algorithm were not well founded, and the method was quickly broken \citet{carlini_attack_2020}.

We present a variant of InstaHide with rigorous privacy guarantees and study its use to rebuff poisoning attacks.  Like the original InstaHide, our approach begins by applying mixup augmention to a dataset.  However instead of introducing randomness through a multiplicative mask, we instead introduce randomness by added Laplacian noise. Our approach exploits the fact that mixup augmentation concentrates training data near the center of the ambient unit hypercube and saturates this region of space more densely than the original dataset. Hence, less noise is required to render the data private than if noise were added to the original data. In fact, we show that adding noise on top of $k$-way mixup creates a differential privacy guarantee that is $k$ times stronger (i.e., $\epsilon$ is $k$ time smaller) than adding noise alone. 

In addition to mixup, we also perform experiments with the related CutMix and MaxUp augmentations.  Because these augmentations are designed for improving generalization in image classifiers, we find that they yield a favorable robustness accuracy trade-off compared to other strong defenses \citep{yun2019cutmix, zhang2017mixup, gong2020maxup}. 

This paper is an extended version of our preliminary short paper, \citet{borgnia2020strong}.  We go beyond that preliminary work to characterize the privacy benefits of mixup theoretically, and greatly extend the empirical analysis of data augmentation defenses against data poisoning attacks.

\subsection{Related Work}
\label{subsec:related}
Broadly speaking, data poisoning attacks aim to compromise the performance of a network by maliciously modifying the data on which the network is trained. Data poisoning attacks vary in their goals, methods, and settings. In general, the goals of a data poisoning attack can be divided into \textit{indiscriminate} attacks, which seek to degrade general test-time performance of a network, and \textit{targeted} attacks, which aim to cause a specific example, or set of examples, to be misclassified \cite{barreno_security_2010}.   

Early work on data poisoning often focused on indiscriminate attacks in simple settings, such as support vector machines, logistic regression models, principle component analysis, or clustering algorithms \cite{munoz2017towards,xiao2015feature, biggio2012poisoning, koh2018stronger}.

However, these early methods do not scale well to modern deep networks \citep{huang2020metapoison}. Many recent works instead focus on targeted attacks and backdoor attacks, which are easier to scale and can be more insidious since they do not lead to any noticeable degradation in validation accuracy, making them harder to detect \citep{geiping2020witches}. 
Accordingly, in this work, we focus on defending against targeted and backdoor attacks. Within these attacks, however, there still exists a wide range of methods and settings. Below, we detail a few categories of attacks. A comprehensive enumeration of backdoor attacks, data poisoning attacks, and defense can be found in \citet{goldblum2020dataset}.

A \textbf{feature collision} attack occurs when the attacker modifies training samples so they collide with, or surround, a target test-time image. \textit{Poison Frogs} \citep{shafahi2018poison} optimizes poisons to minimize the $\ell_2$ distance in feature space between the poisoned and target features, while also including a regularization term on the size of the perturbations. Newer methods like \textit{convex polytope} \citep{zhu2019transferable} and \textit{bullseye polytope} \citep{aghakhani_bullseye_2020} surround the target image in feature space to improve the stability of poisoning. All these methods work primarily in the transfer learning setting, where a known feature extractor is fixed and a classification layer is fine-tuned on the perturbed data. 

\textbf{From-scratch} attacks modify training data to cause targeted misclassification of pre-selected test time images. Crucially, these attacks work in situations where a deep network is \textit{a priori} trained on modified data, rather than being pre-trained and subsequently fine-tuned on poisoned data. \textit{MetaPoison} \citep{huang2020metapoison} optimizes poisons by unrolling training iterations to solve a bi-level optimization problem.  \textit{Witches' Brew} \citep{geiping2020witches} approximately solve the bi-level optimization problem using a gradient alignment objective. 

\textbf{Backdoor attacks}
involve inserting a ``trigger," often a fixed patch, into training data. Attackers can then add the same patch to data at test time to fool the network into misclassifying modified images as the target class. Some forms of backdoor attacks will patch a number of training images with a small pattern or even modify just a single pixel \cite{gu2017badnets, tran2018spectral}. More complex attacks, like hidden-trigger backdoor \cite{saha2020hidden}, adaptively modify the training data to increase the success of the additive patch at test time. 

Conversely, a variety of defenses against poisoning attacks have also been proposed. Many defenses to targeted poisoning attacks can broadly be classified as \textit{filtering defenses}, which either remove or relabel poisoned data. These methods rely on the tendency of poisoned data to differ sufficiently from clean data in feature space. Intuitively, one could use a pretrained network as a feature extractor to sort out poison from the clean data.
Once the poisoned data is found and isolated in feature space, it is removed from the dataset and the model is retrained from scratch. Conveniently, filtering defenses do not require any external source of trusted clean data and work even if the feature extractor is trained on poisoned data. 

Among filtering defenses, {\em Spectral Signatures} \cite{tran2018spectral,paudice2018label}  filter data based on which points have the highest correlation with the top right singular vector of the feature covariance matrix.
{\em Activation Clustering} \cite{acchen2018detecting} instead uses $k$-means clustering to separate feature space, relying on the heuristic that poisons tend to cluster in feature space.
{\em DeepKNN} \cite{peri2019deep} relabels outlier data in feature space according to a $k$-nearest neighbors algorithm, hoping to diminish the effects of poison using the same heuristic that poisoned data are outliers in feature space. 
Unfortunately,
filtering defenses have proven weak against more advanced attacks, especially in the from-scratch setting \cite{geiping2020witches}, and may be nullified by adaptive attacks that carefully circumvent detection \citep{koh2018stronger}. 

Certified defenses avoid the possibility of breaking under adaptive attacks using robust mechanisms such as randomized smoothing or by partitioning the training data and individually training classifiers on each partition \citep{weber2020rab, levine2020deep}. 

Another class of principled defenses use differentially private SGD, where training gradients are clipped and noised thus diminishing the effects of poisoned gradient updates. However, these defenses have been shown to fail against advanced attacks, as they often lead to significant drops in clean validation accuracy \cite{geiping2020witches}.  

Outside of data poisoning, \citet{lee2019synthesizing} connect data augmentation and privacy by using tools for R\'enyi differential privacy for subsampling \citep{wang2019subsampled} to analyze R\'enyi bounds for image mixtures with Gaussian noise. While these bounds can readily be converted into differential privacy guarantees, they suffer from numeric instability and tend to be loose in the low privacy regime, where validation accuracy is maintained.

\section{Data Augmentation as an Empirical Defense against Dataset Manipulation}
\label{sec:experiments}
\begin{figure*}
\centering
\includegraphics[width=0.8\textwidth]{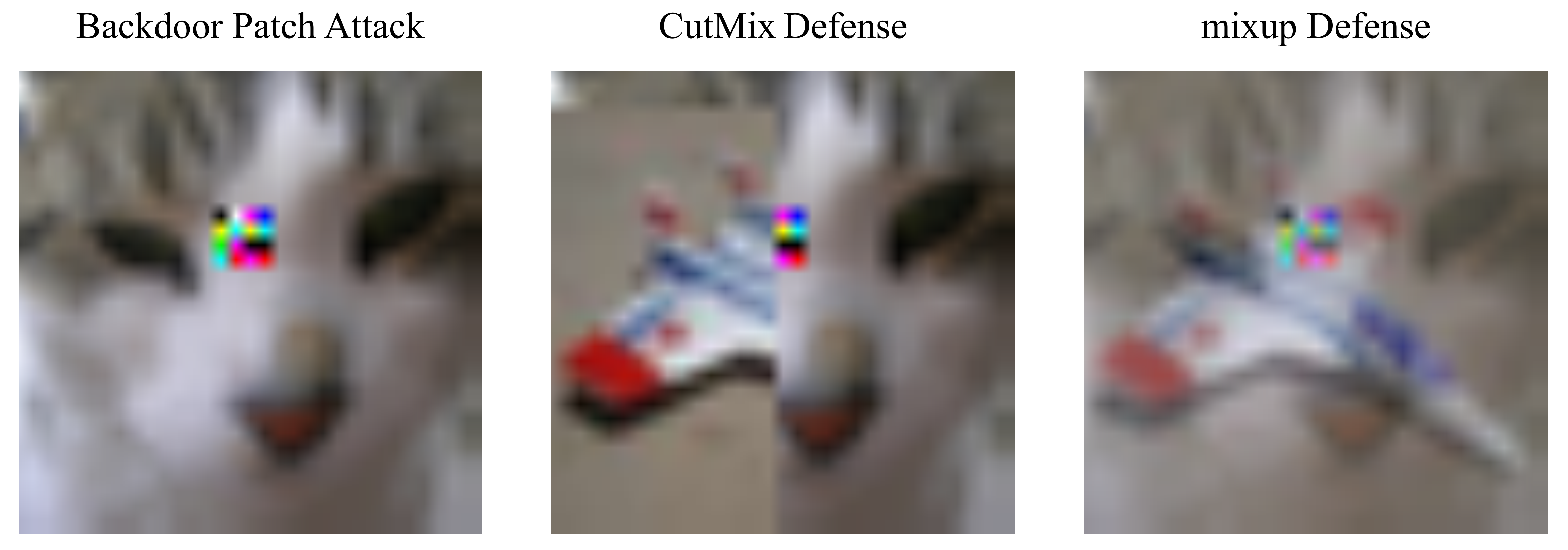}
\caption{``Cat'' image from CIFAR-10 with a backdoor patch and the same image with CutMix and mixup augmentations.}
\label{fig:augmentations}
\end{figure*}

Before studying the provable benefits of mixup, we study the empirical effectiveness of data augmentations to prevent poisoning. We are mainly interested in data augmentations that mix data points; we consider the hypothesis that data poisoning attacks rely on the deleterious effects of a subset of modified samples, which can in turn be diluted and deactivated by mixing them with other, likely unmodified, samples. 

One such augmentation is mixup, proposed in \citet{zhang2017mixup}, which trains on samples $(x,y)_{i=1}^k$ mixed randomly in input space
\begin{equation}
    \hat{x} = \sum_{i=1}^k \lambda_i x_i, \quad \hat{y} = \sum_{i=1}^k \lambda_i y_i,
\end{equation}
to form the augmented sample $(\hat{x},\hat{y}).$ Though $\lambda$ is traditionally drawn from a Dirichlet distribution parametrized by some chosen factor $\alpha$, we will restrict to the case of equal weighting $\lambda = 1/k$ to aid in theoretical analysis. From here on, $k$ is referred to as the mixture width.

CutOut \citep{devries2017improved}, which blacks out a randomly generated patch from an image, can be combined with mixup to form CutMix \citep{yun2019cutmix}, another type of mixing augmentation. Specifically, the idea is to paste a randomly selected patch from one image onto a second image, with labels computed by taking a weighted average of the original labels. The weights of the labels correspond to the relative area of each image in the final augmented data point.

MaxUp \citep{gong2020maxup} can also be considered as a mixing data augmentation, which first generates augmented samples using various techniques and then selects the sample with the lowest associated loss value to train on. CutMix and mixup will be the central mixing augmentations that we consider in this work, which we contrast with MaxUp in select scenarios.  

Adding noise to input data is another augmentation method, which can be understood as a mixing augmentation that combines input data not with another image, but with a random sample from the input space, unrelated to the data distribution. This mechanism is also common in differential privacy \citep{hardt2010geometry}. Since the exact original image is not apparent from its noised counterpart, adding noise decreases the sensitivity of the new data to the original dataset.  We will return to the subject of additive noise when we discuss the connection between data augmentation and differential privacy guarantees.

\subsection{Backdoor Attacks} 

\par In contrast to recent targeted data poisoning attacks, \textit{backdoor} attacks often involve inserting a simple preset trigger into training data to cause base images to be misclassified into the target class. For our experiments, we use small $4 \times 4$ randomly generated patches as triggers to poison the target class (See Figure~\ref{fig:augmentations}). To evaluate the baseline effectiveness of backdoor attacks, we poison a target class, train a ResNet-18 model on this poisoned data and use it to classify patched images from a victim test class. Only if a patched image from a victim class is labeled with the target class do we treat it as a successfully poisoned example.  Our results show that backdoor attacks achieve $98.3\%$ poison success when $100\%$ of images from the target class are poisoned and $45.6\%$ poison success when only $10\%$ of target images are patched (see Table~\ref{tab:backdoor}). In addition, when $100\%$ of training images from the target class are patched, clean test accuracy of the model drops by almost 10\% since the model is unable to learn meaningful features of the target class. \\

\par  We then compare the baseline model to models trained with the mixup and CutMix data augmentation techniques. We find that although mixup helps when only part of the target class is poisoned, it is not efficient as a defense against backdoor attacks when all images in the target class are patched. In contrast, CutMix is an extremely effective defense against backdoor attacks in both scenarios and it reduces poison success from 98.3\% to 14.1\% in the most aggressive setting. Finally, models trained on poisoned data with CutMix data augmentation have a clean test accuracy similar to the accuracy of models trained on clean data. Intuitively, CutMix often produces patch-free mixtures of the target class with other classes, hence the model does not solely rely on the patch to categorize images of this class.

We extend this analysis to two more complex attacks, clean-label backdoor attacks 
\cite{turner2018clean}, and hidden-Trigger backdoor attacks in Table~\ref{tab:backdoor_defense}.

\begin{table*}[t]
\vskip 0.15in
\caption{Validation accuracy and poison success for a baseline model, models trained with mixup and CutMix augmentations (rows 2,3) and Spectral Signature \cite{sstran2018spectral} and Activation Clustering \cite{acchen2018detecting} defenses (rows 4,5). The first two columns correspond to the case where 10\% of one class is poisoned. The last two columns correspond to the case where all images of one class are poisoned (a scenario in which filter defenses are inapplicable as no unmodified images remain for this class). The results are averaged across 20 runs (with different pairs of target and victim classes). }
\label{tab:backdoor}
\begin{center}
\begin{small}
\begin{sc}
\begin{tabular}{lcccc}
\toprule
\textbf{} &  Clean Accuracy (10\%) & Poison Success (10\%) & Clean Accuracy (100\%) & Poison Success (100\%) \\\\
\midrule
Baseline & 94.3\% & 45.6\% & 85.0\% & 98.3\%\\
\midrule
CutMix & 95.1\% & \bf{7.0\%} & 94.2\% & \bf{14.1}\%\\
mixup & 94.4\% & 23.9\% & 85.3\% & 99.8\%\\
\midrule 
SS & 92.3\% & 48.3\% & & \\
AC & 89.4\% & 44.0\% & & \\

\bottomrule
\end{tabular}
\end{sc}
\end{small}
\end{center}
\vskip -0.1in
\end{table*}

\subsection{Targeted Data Poisoning}
We further evaluate data augmentations as a defense against targeted data poisoning attacks. We analyze the effectiveness of CutMix and mixup as a defense against feature collision attacks in Table~\ref{tab:fc_defense}. Applying these data augmentations as a defense against Poison Frogs \citep{shafahi2018poison} (FC) is exceedingly successful, as the poisoned data is crafted independently there, making it simple to disturb by data augmentations. The poisons crafted via Convex Polytope (CP) \citep{zhu2019transferable} however, are more robust to data augmentations, due to the polytope of poisoned data created around the target.  Nonetheless, the effectiveness of CP is diminished more by data augmentations than by other defenses.

We then evaluate the success of data augmentations against Witches' Brew, the gradient matching attack of \citet{geiping2020witches} in Table~\ref{tab:witchesbrew_defenses}. Against this attack, we evaluate a wide range of data augmentations, as the attack is relatively robust to basic mixup data augmentations which mix only two images. However, using a stronger augmentation that mixes four images still leads to a strong defense in the non-adaptive setting (where the attacker is unaware of the defense). As this attack can be adapted to specific defenses, we also consider such a scenario. Against the adaptive attack, we found MaxUp to be most effective, evaluating the worst-case loss for every image in a minibatch over four samples of data augmentation drawn from cutout. To control for the effects of the CIFAR-10 dataset that we consider for most experiments, we also evaluate defenses against an attack on the ImageNet dataset in Table~\ref{tab:witchesbrew_defenses_imagenet}, finding that the described effects transfer to other datasets.

\begin{table}[!ht]
\caption{Poison success rates (lower is better for the defender) for various data augmentations tested against the gradient matching attack of \citet{geiping2020witches}. 
All results are averaged over 20 trials. We report the success of both a non-adaptive and an adaptive attacker.}

\vspace{-2mm}
\label{tab:witchesbrew_defenses}
\vskip 0.15in
\begin{center}
\begin{small}
\begin{sc}
\begin{tabular}{l| ccccc}
\toprule\\
Augmentation & Non-Adaptive & Adaptive \\
\midrule
2-way mixup & 45.00\% & 72.73\% \\
Cutout & 60.00\% & 81.25\% \\
CutMix & 75.00\% & 60.00\% \\
4-way mixup & 5.00\% &  55.00\% \\
MaxUp-Cutout &5.26\%  & 20.00\% \\
\bottomrule
\end{tabular}
\end{sc}
\end{small}
\end{center}
\vskip -0.1in
\end{table}

\vspace{-2mm}
\begin{table}[!ht]
\caption{Success rate for selected data augmentation when tested against the gradient matching attack on the ImageNet dataset. All results are averaged over 10 trials.}
\label{tab:witchesbrew_defenses_imagenet}
\vskip 0.15in
\begin{center}
\begin{small}
\begin{sc}
\begin{tabular}{l| ccccc}
\toprule\\
Augmentation & Poison success \\
\midrule
None & 90\% \\
2-way mixup & 50.00\% \\
4-way mixup & 30.00\% \\
\bottomrule
\end{tabular}
\end{sc}
\end{small}
\end{center}
\vskip -0.1in
\end{table}

\subsection{Comparison to Other Defenses} 
We compare our method to previous defenses referenced in Section \ref{subsec:related}.  
We show that our method outperforms filter defenses when evaluating backdoor attacks, such as in Table~\ref{tab:backdoor} and Table~\ref{tab:backdoor_defense}, as well as when evaluating targeted data poisoning attacks, as we show for Poison Frogs and Convex Polytope in Table~\ref{tab:fc_defense} and for Witches' Brew in Table~\ref{tab:witchesbrew_defenses_imagenet} and ~\ref{tab:witchesbrew_competing_defenses}.
We note that data augmentations do not require additional training compared to filter defenses in some settings and are consequently more computationally efficient.

In Figure~\ref{fig:wb_tradeoff}, we plot the average poison success against the validation error for adaptive gradient matching attacks. We find that data augmentations exhibit a stronger security performance trade-off compared to other defenses.

%



\begin{table}[!ht]
\caption{Poison success rate for Poison Frogs \citep{shafahi2018poison} and Convex Polytope \citep{zhu2019transferable} attacks when tested with baseline settings and when tested with mixup and CutMix. All results are averaged over 20 trials.}
\label{tab:fc_defense}
\vskip 0.15in
\begin{center}
\begin{small}
\begin{sc}
\begin{tabular}{lccccc}
\toprule\\
Attack & Baseline & SS & AC & mixup & CutMix \\\\
\midrule
FC & 80\% & 70\% & 45\% & \textbf{5\%} & \textbf{5\%} \\
\midrule
CP & 95\% & 90\% & 75\% & 70\% & \textbf{50}\% \\
\bottomrule
\end{tabular}
\end{sc}
\end{small}
\end{center}
\vskip -0.1in
\end{table}

\begin{table}[!ht]
\caption{Poison success rates (lower is better for the defender) for competing defenses when tested against the gradient matching attack compared to mixup. For DP-SGD, we consider a noise level of $n=0.01$. All results are averaged over 20 trials.}
\label{tab:witchesbrew_competing_defenses}
\vskip 0.15in
\begin{center}
\begin{small}
\begin{sc}
\begin{tabular}{l| c}
\toprule
Defense & Poison Success \\
\midrule
Spectral Signatures & 95.00\% \\
deepKNN & 90.00\% \\
Spectral Signatures & 95.00\% \\
Activation Clustering & 30.00\% \\
\midrule
DP-SGD & 86.25\%  \\
\midrule 
4-way mixup & 5.00\% \\
\bottomrule
\end{tabular}
\end{sc}
\end{small}
\end{center}
\vskip -0.1in
\end{table}

\begin{table}[!ht]
\caption{Success rate against backdoor attacks when tested with baseline settings and when tested with the mixup and CutMix. All results are averaged over 20 trials.}
\label{tab:backdoor_defense}
\vskip 0.15in
\begin{center}
\begin{small}
\begin{sc}
\begin{tabular}{lccccc}
\toprule\\
Attack & Baseline & SS & AC & mixup & CutMix \\\\
\midrule
HTBD & 60\% & 65\% & 55\% & 20\% & \textbf{10\%} \\
\midrule
CLBD & 65\% & 60\% & 45\% & 25\% & \textbf{15\%} \\

\bottomrule
\end{tabular}
\end{sc}
\end{small}
\end{center}
\vskip -0.1in
\end{table}


\begin{figure*}
\vskip 0.2in
\begin{center}
\centerline{\includegraphics[width=0.75\textwidth]{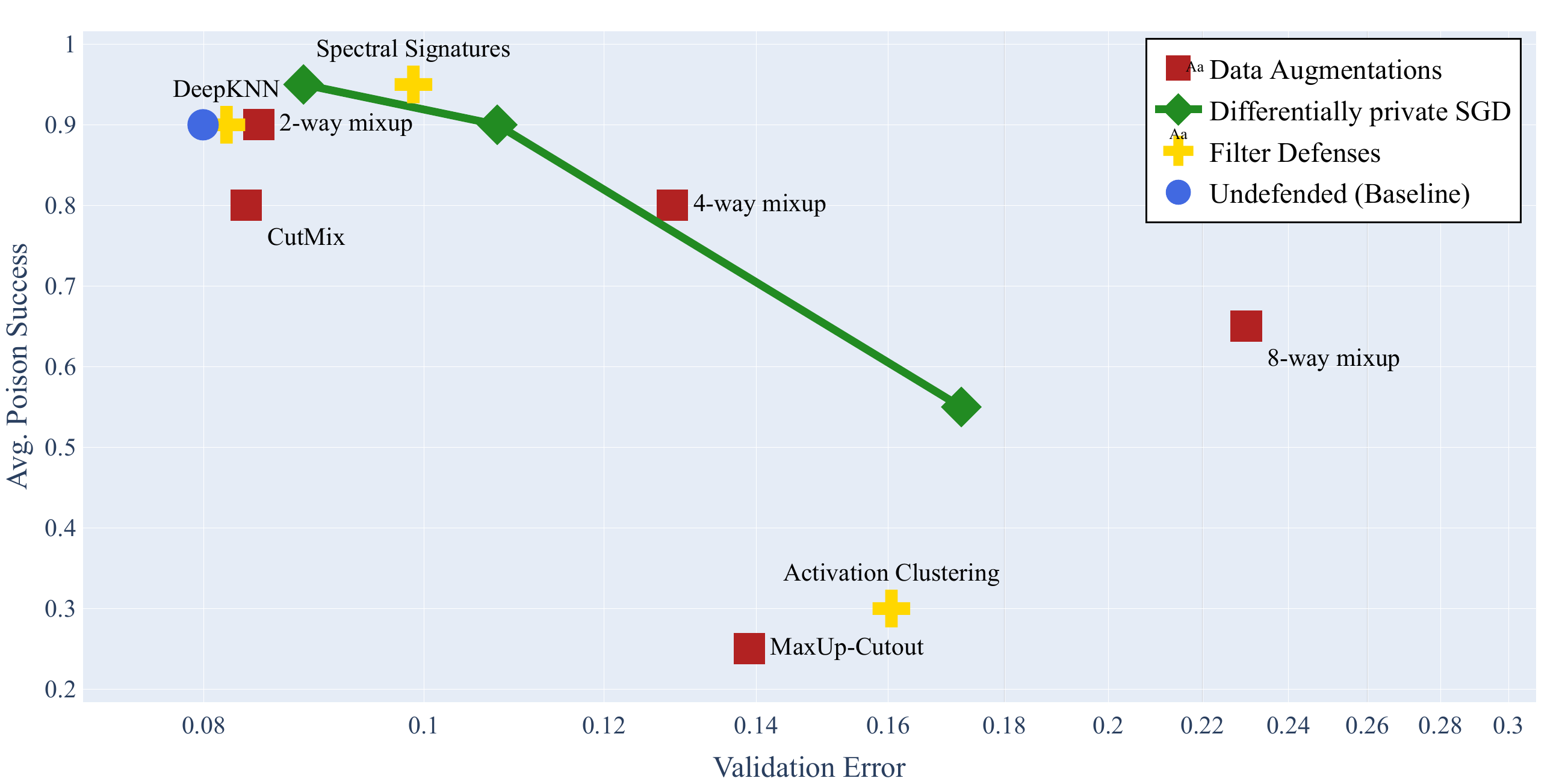}}
\caption{Trade-off between average poison success and validation accuracy for various defenses against gradient matching (adaptive).}
    \label{fig:wb_tradeoff}
\end{center}
\vskip -0.2in
\end{figure*}

\section{DP-InstaHide: A Mixup Defense with Provable Differential Privacy Advantages}
\label{sec:provable}
The original InstaHide method \cite{huang_instahide_2020} attempted to privatize data by first applying mixup, and then multiplying the results by random binary masks. While the idea that mixup enhances the privacy of a dataset is well founded, the original InstaHide scheme lies outside of the classical differential privacy framework, and is now known to be insecure \cite{carlini_attack_2020}.  We propose a variant of the method, DP-InstaHide, which replaces the multiplicative random mask with additive random noise. The resulting method comes with a differential privacy gaurantee that enables us to quantify and analyze the privacy benefits of mixup augmentation.


Differential privacy, developed by \citet{dwork2014algorithmic}, aims to prevent the leakage of potentially compromising information about individuals present in released data sets. By utilizing noise and randomness, differentially private data release mechanisms are provably robust to any auxiliary information available to an adversary. 

Formally, let $\mathcal{M}: \mathcal{D} \rightarrow \mathcal{R}$ be a random mechanism, mapping from the space of datasets to a co-domain containing potential outputs of the mechanism. We consider a special case where $\mathcal{R}$ is another space of datasets, so that $\mathcal{M}$ outputs a synthetic dataset. We say two datasets $D,D' \in \mathcal{D}$ are adjacent if they differ by at most one element, that is $D'$ has one fewer, one more, or one element different from $D$.

Then, $\mathcal{M}$ is $(\epsilon,\delta)$-differentially private if it satisfies the following inequality for any $U \subseteq \mathcal{R}$:
\begin{equation}
  \mathbb{P}[\mathcal{M}(D) \in U] \leq e^{\epsilon}\mathbb{P}[\mathcal{M}(D') \in U] + \delta.
\end{equation} 

Intuitively, the inequality and symmetry in the definition of dataset adjacency tells us that the probability of getting any outcome from $\mathcal{M}$ does not strongly depend on the inclusion of any individual in the dataset. In other words, given any outcome of the mechanism, a strong privacy guarantee implies one cannot distinguish whether $D$ or $D'$ was used to produce it. This sort of indistinguishability condition is what grants protection from linkage attacks such as those explored by \citet{narayanan2006break}. The quantity $\epsilon$ describes the extent to which the probabilities differ for \textit{most} outcomes, and $\delta$ represents the probability of observing an outcome which \textit{breaks} the $\epsilon$ guarantee.

\begin{figure*}[t]
\centering
\includegraphics[width=\textwidth]{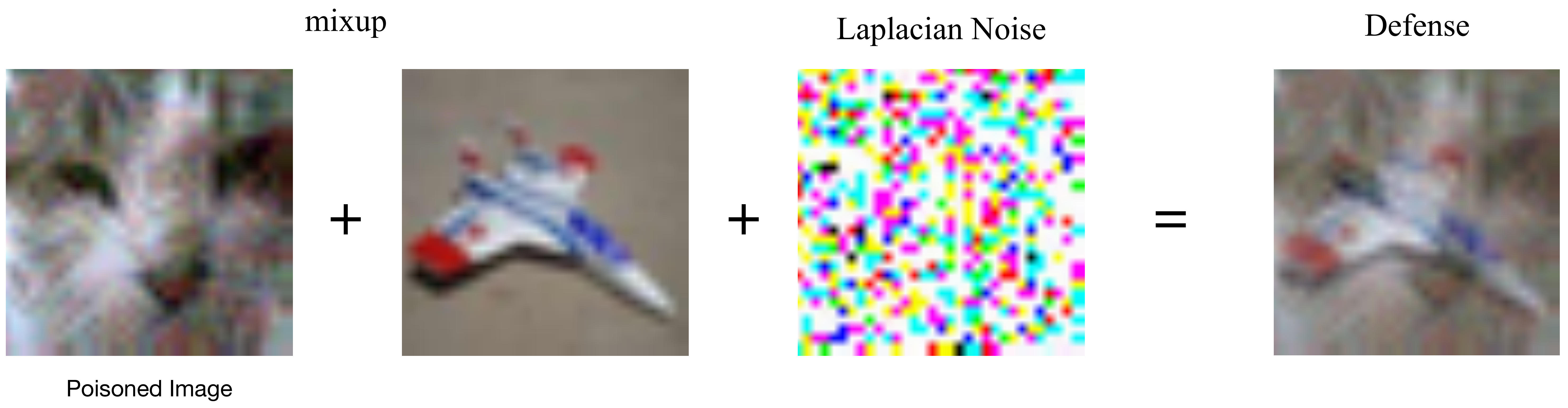}
\caption{Illustration of the DP-InstaHide defense on two CIFAR-10 images, the first of which has been poisoned with $\varepsilon=16$.  Mixup is used to average two images, and then Laplacian noise is added,}
\label{fig:mixup_noise}
\end{figure*}

In the case where differentially private datasets are used to train neural networks, such indistinguishability also assures poisoned data will not have a large effect on the trained model. \citet{ma2019data} formalize this intuition by proving a lower bound for the defensive capabilities of differentially private learners against poisoning attacks.

We define the threat model as taken from \citet{ma2019data}: The attacker aims to direct the trained model $\mathcal{M}(D')$ to reach some attack target by modifying at most $l$ elements of the clean dataset $D$ to produce the poisoned dataset $D'$.  We measure the distance of $\mathcal{M}(D')$ from the attack target using a cost function $C$, which takes trained models as an input and outputs an element of $\mathbb{R}$. The attack problem is then to minimize the expectation of the cost of $\mathcal{M}(D')$.
\begin{equation}
    \min_{D'} J(D') := \mathbb{E}[C(\mathcal{M}(D'))]
\end{equation}
 
Finally, we arrive at the theorem proven in \citet{ma2019data}.
\begin{theorem}
For an $(\epsilon,\delta)$-differentially private mechanism $\mathcal{M}$ and bounded cost function $|C| \leq B$, it follows that the attack cost $J(D')$ satisfies

\begin{equation}
    J(D') \geq \max\{e^{-l\epsilon}\left(J(D)+\frac{B\delta}{e^{\epsilon}-1}\right)-\frac{B\delta}{e^{\epsilon}-1}, 0\}
\end{equation}

\begin{equation}
    J(D') \geq \max\{e^{-l\epsilon}\left(J(D)+\frac{B\delta}{e^{\epsilon}-1}\right)+\frac{B\delta}{e^{\epsilon}-1}, -B\}
\end{equation}

where the former bound holds for non-negative cost functions and the latter holds for non-positive cost functions.
\end{theorem}

Empirically, however, it is found that the defense offered by differential privacy mechanisms tends to be more effective than the theoretical limit. Likely, this is a result of differential privacy definitionally being a worst-case guarantee, and in practice the worst case is rarely observed. 

We find that differential privacy achieved through the combination of $k$-way mixup and additive Laplacian noise is an example of such a defense, practically visualized in Fig.\ref{fig:mixup_noise}. Because mixup augmentation concentrates training data near the center of the unit hypercube, less noise must be added to the mixed up data to render the noisy data indistinguishable from other points nearby in comparison to solely adding noise to the data points \citep{zhang2017mixup}. Additionally, mixup benefits from improved generalization due to its enforcement of linear interpolation between classes and has recently been shown to be robust to a variety of adversarial attacks, such as FGSM \cite{zhang2020does}. We use a combinatorial approach to achieve a formal differential privacy guarantee for mixup with Laplacian noise, which in tandem with the result from \citet{ma2019data} gives us a direct theoretical protection from data poisoning.

\begin{figure*}[t]
    \centering
    \begin{subfigure}[t]{0.49\textwidth}
    \includegraphics[width=\textwidth]{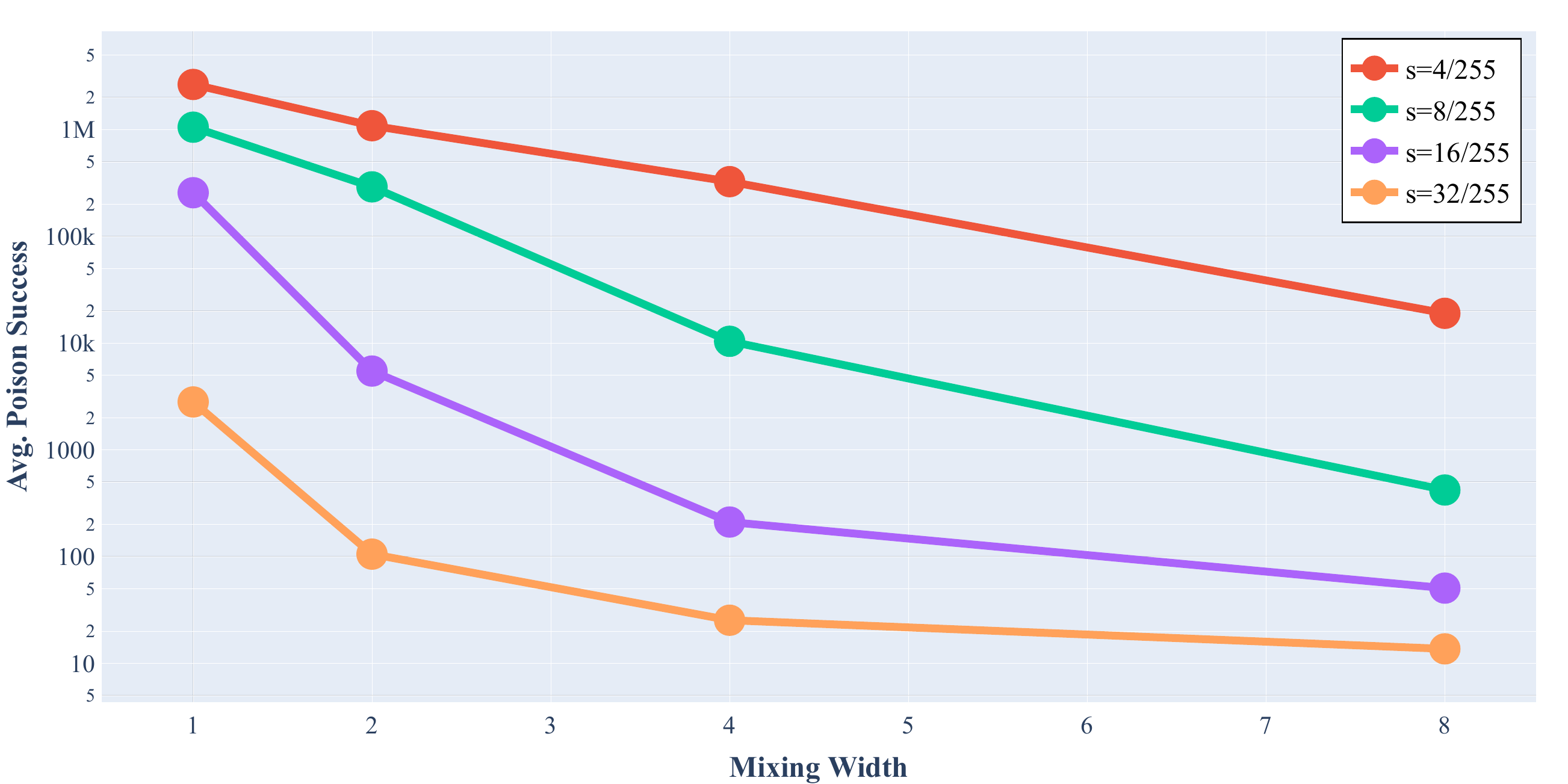}
    \caption{Theoretical privacy guarantees}
    \label{fig:bounds:theoretical}
    \end{subfigure}
    \begin{subfigure}[t]{0.49\textwidth}
    \includegraphics[width=\textwidth]{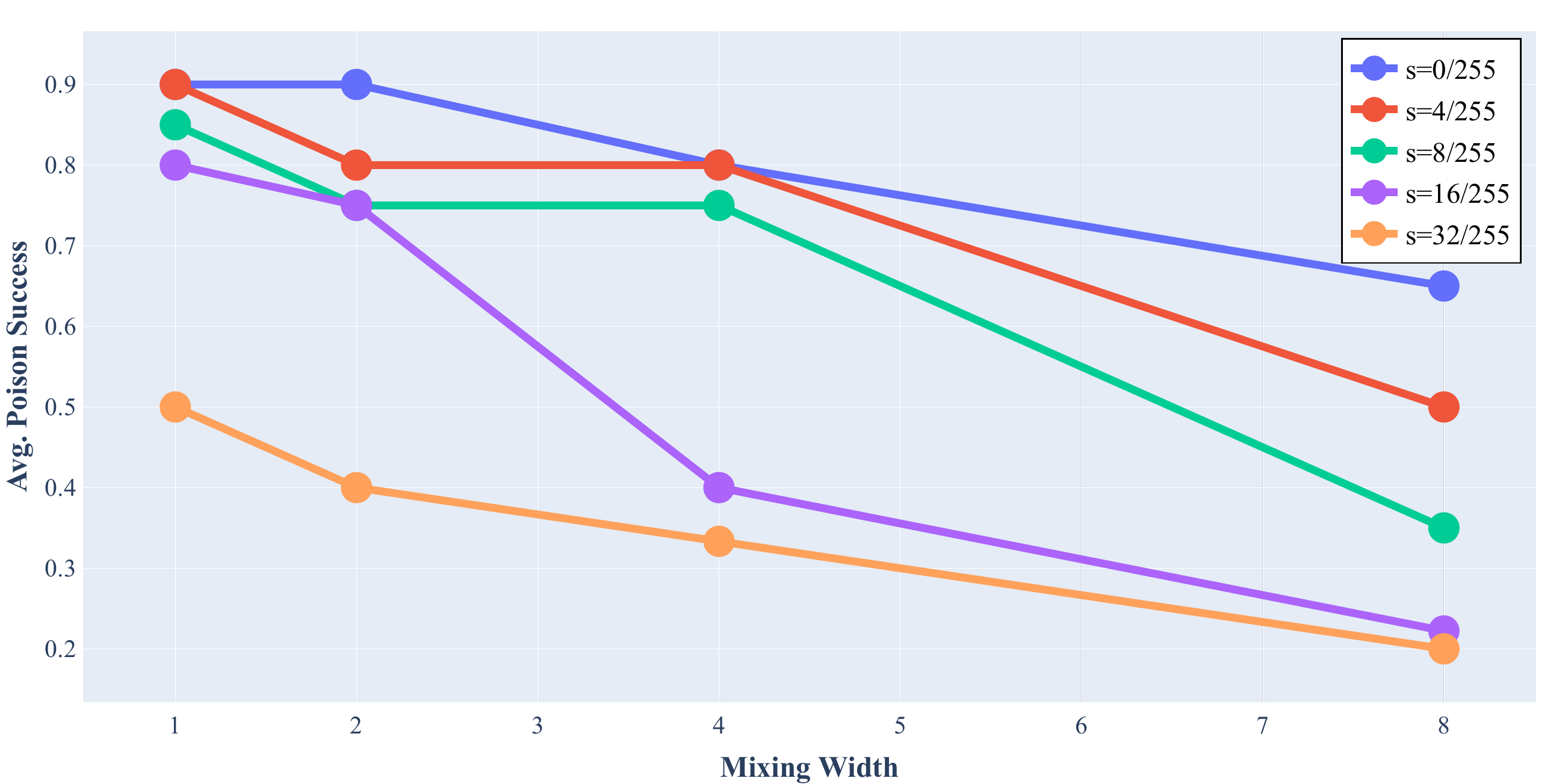}
    \caption{Empirical defense against poisoning attacks}
    \label{fig:bounds:empirical}
    \end{subfigure}
    \caption{Theoretical and empirical mixup. Left: Privacy guarantee $\epsilon$ as a function of mixture width $k$, computed for each implemented Laplacian noise level $s$. We use values $n = T = 5 \times 10^4$, corresponding to the CIFAR-10 dataset. Right: Poisoning success for a strong adaptive gradient matching attack for several mixture widths and noise levels.}
    \label{fig:bounds}
\end{figure*}
\subsection{A Theoretical Guarantee for DP-InstaHide}

Above, we discussed how strong data augmentations, such as mixup and random noise, provide an empirically strong defense against poisoning. We can explain the strength of this defense, and provide a rigorous guarantee, by analyzing the privacy benefits of mixup within a differential privacy framework.

Let $D$ be a dataset of size $n$ and $D'$ denote the same dataset with the point $x_0$ removed. Let $d$ be the dimension of data points and assume the data lies in a set $V$ of diameter one, i.e., $sup\{||D - D^\prime||_1 : D, D^\prime \in V\}\le 1$. We sample a point of the form $z = \frac{1}{k}(x_1 + x_2 + \cdots + x_k) + \eta$, where the $x_i$ are drawn at random from the relevant dataset $P$ without replacement, and $\eta \sim Lap(\mathbf{0},\sigma I)$ is the independent $d$-dimensional isotropic Laplacian additive noise vector with density function $\phi_{\sigma}(\eta) = \frac{1}{(2\sigma)^d}e^{\|\eta\|_1/\sigma}.$
The random variable representing the outcome of the sampling is therefore a sum of random variables:
\begin{equation}
    \mathcal{M}_P = \frac{1}{k}\sum_{i=1}^k X_i + N  
\end{equation}
We use $p$ and $q$ to denote the probability density functions of $\mathcal{M}_D$, and $\mathcal{M}_{D'}$ respectively.

\begin{theorem}
\label{thm:mixup_privacy}
Assume the data set $D$ has $\ell_1$-norm radius less than 1, and that mixup groups of mixture width $k$ are sampled without replacement. The mixup plus Laplacian noise mechanism producing a data set of size $T$ satisfies $(\epsilon, 0)$-differential privacy with
$$\epsilon = T\max \left\{A,B \right\} \le \frac{T}{k\sigma}$$
where
$$A = \log \left( 1- \frac{k}{n} +  e^{\frac{1}{k\sigma}} \frac{k}{n}\right), \,\, B = \log\frac{n}{n-k+k e^{-\frac{1}{k\sigma}}}.$$



\end{theorem}

\begin{proof}
To prove differential privacy, we must bound the ratio of $\mathbb{P}[\mathcal{M}_D \in U]$ to $\mathbb{P}[\mathcal{M}_{D'} \in U]$ from above and below, where $U \subseteq V$ is arbitrary and measurable. For a fixed sampling combination $x = (x_1, \dots, x_k) \in D^k$, the density for 
observing $z=\frac{1}{k}\sum_{i=1}^k x_i + N$ is given by $\phi_{\sigma}\left(z-\sum_{i=1}^k x_i\right)$.  Since there are ${n \choose k}$ possible values that $x$ can take on, each of equal probability, we have
$$p(z) =  \frac{k!(n-k)!}{n!}  \sum_{x \in D^k} \phi_{\sigma}\left(z-\sum_{i=1}^k x_i\right).$$

Let's now write a similar expression for $q(z)$. We have 
\begin{equation}
q(z) =\frac{k!(n-k-1)!}{(n-1)!} \sum_{x \in {D'}^k} \phi_{\sigma}\left(z-\sum_{i=1}^k x_i\right).\label{defq}
\end{equation}

Now, we write the decomposition $p(z)=p_0(z)+p_1(z)$, where $p_0(z)$ is the probability of the ensemble not containing $x_0$ times the conditional density for observing $z$ given this scenario, and $p_1(z)$ is the probability of
having $x_0$ in the ensemble times the conditional density for observing $z$ given this scenario.

Then, we have
\begin{equation}
    p_0(z) = \left( 1- \frac{k}{n}\right)q(z). \label{pzero}
\end{equation}

Now, consider $p_1(z).$  This can be written
\begin{equation}\label{eq:beforeswap}
\small
p_1(z) = \frac{k}{n}  \frac{(k-1)!(n-k-2)!}{(n-1)!} \!\!\! \! \!\sum_{x \in {D'}^{k-1}}\!\!\!\!\! \phi_{\sigma}\left(z-x_0- \sum_{i=1}^{k-1} x_i\right). 
\end{equation}
In the equation above, $\frac{k}{n} $ represents the probability of drawing an ensemble $x$ that contains $x_0,$ and the remainder of the expression is the probability of forming $z-x_0$ using the remaining $k-1$ data points in the ensemble. 

We can simplify equation \eqref{eq:beforeswap} using a combinatorial trick.  Rather than computing the sum over all tuples of size $k-1,$ we compute the sum over all tuples of length $k,$ but we discard the last entry of each tuple.  We get 
\begin{equation}\label{eq:afterswap}
p_1(z) = \frac{k}{n}  \frac{k!(n-k-1)!}{(n-1)!} \! \! \sum_{x \in {D'}^{k}} \phi_{\sigma}\left(z-x_0- \sum_{i=1}^{k-1} x_i\right).
\end{equation}

Now, from the definition of the Laplace density, we have that if $\|u-v\|_1<\epsilon$ for any $u,v$ then
$$e^{-\|u-v\|_1/\sigma}\phi_{\sigma}(v) \le \phi_{\sigma}(u)\le e^{\|u-v\|_1/\sigma}\phi_{\sigma}(v).$$
Let's apply this identity to \eqref{eq:afterswap} with
$u= z-x_0 - \sum_{i=1}^{k-1} x_i$ and $v= z- \sum_{i=1}^{k} x_i$. We get
$$e^{-\frac{1}{k\sigma}} \frac{k}{n}  q(z) \le p_1(z) \le e^{\frac{1}{k\sigma}}  \frac{k}{n} q(z), $$
where we have used the fact that the dataset $D$ has unit diameter to obtain $\|u-v\|_1\le \frac{j}{k},$ and we used the definition \eqref{defq} to simplify our expression.
\begin{figure*}[t]
    \centering
    \includegraphics[width=0.75\textwidth]{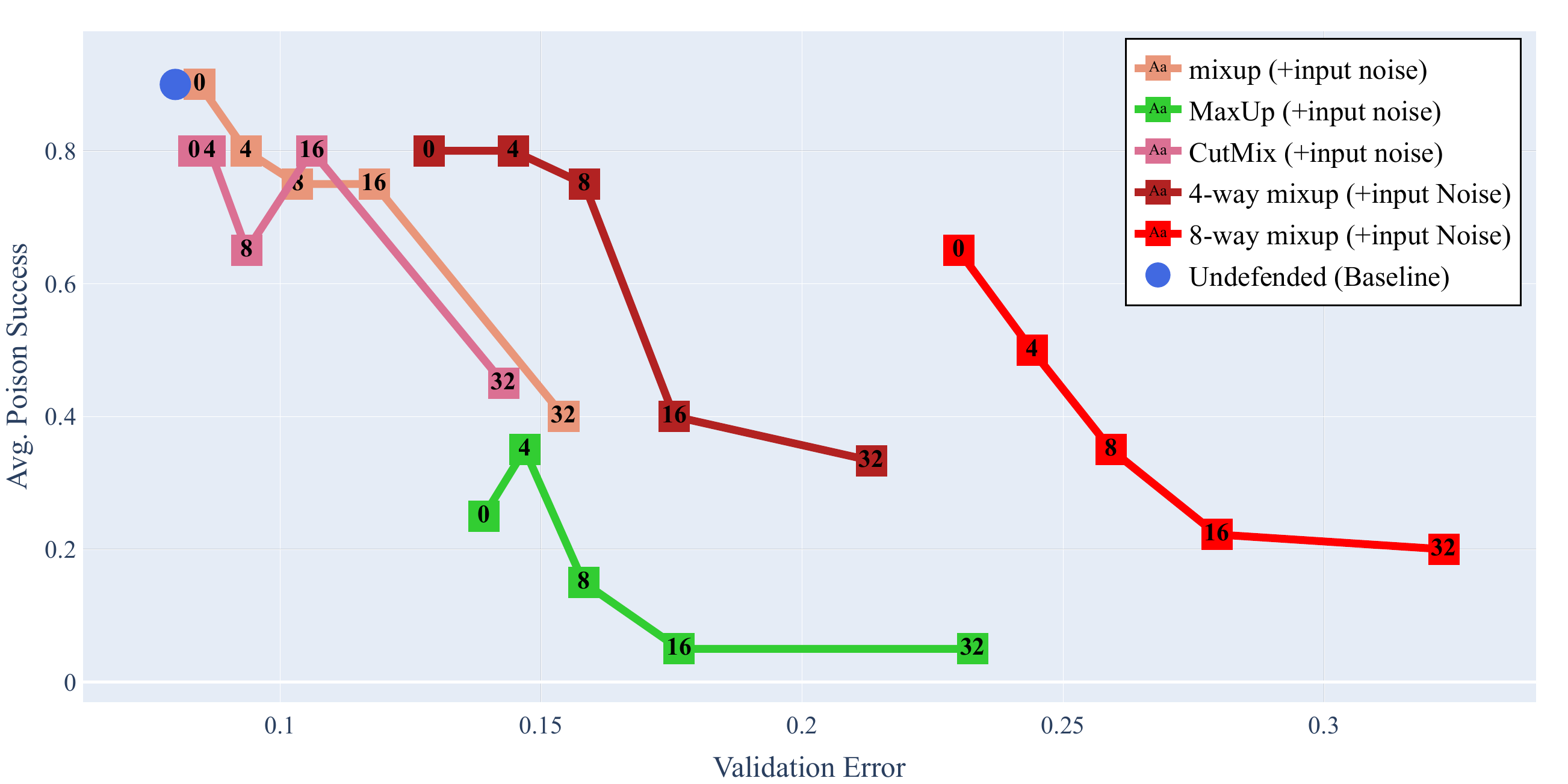}
    \caption{Enhancing various data augmentations with Laplacian noise. We visualize the security-performance trade-off when enhancing the data augmentations considered in Sec.~\ref{sec:experiments} with Laplacian noise as predicted by Thm.~\ref{thm:mixup_privacy}. We visualize the development of these data augmentations when adding Laplacian noise with scales ($2/255$, $4/255$, $8/255$, $16/255$, $32/255$).}
    \label{fig:variations_mixing}
\end{figure*}

Now, we add \eqref{pzero} to this equation.  We get
\begin{align*}
\left( 1- \frac{k}{n} +  e^{-\frac{1}{k\sigma}} \frac{k}{n}\right)&  q(z) \le p(z) \\
 &\le  \left( 1- \frac{k}{n} +  e^{\frac{1}{k\sigma}} \frac{k}{n}\right)  q(z).
\end{align*}
From this, we arrive at the conclusion
\begin{align*}
\frac{p(z)}{q(z)} \le  
 \left( 1- \frac{k}{n} +  e^{\frac{1}{k\sigma}} \frac{k}{n}\right) \le e^{\frac{1}{k\sigma}},
\end{align*}
and
\begin{align*}
\frac{q(z)}{p(z)} \le  
 \frac{n}{n-k+k e^{-\frac{1}{k\sigma}}} \le e^{\frac{1}{k\sigma}}.
\end{align*}
The left-most upper bound in the above equation is achieved by replacing $k$ with $n$ wherever $k$ appears outside of an exponent.
We get the final result by taking the log of these bounds and using the composibility property of differential privacy to account for the number $T$ of points sampled.

\end{proof}

{\em Remark:} A classical Laplacian mechanism for differentially private dataset release works by adding noise to each dataset vector separately and achieves privacy with $\epsilon=\frac{1}{\sigma}$. Theorem \ref{thm:mixup_privacy} recovers this bound in the case $k=1,$ however it also shows that $k$-way mixup enhances the privacy guarantee over the classical mechanism \emph{by a factor of at least $k$}.


\subsection{Defending with DP Augmentations in Practice}

We investigate the practical implications of Theorem~\ref{thm:mixup_privacy} in Figure~\ref{fig:bounds}, where we show the predicted theoretical privacy guarantees in Figure~\ref{fig:bounds:theoretical} and the direct practical application for defenses against data poisoning in Figure~\ref{fig:bounds:empirical}. Figure~\ref{fig:bounds:empirical} shows the average poison success for a strong, adaptive gradient matching attack against a ResNet-18 trained on CIFAR-10 (the setting considered in \citet{geiping2020witches} with an improved adaptive attack). We find that the theoretical results predict the success of a defense by mixup with Laplacian noise surprisingly well.

As a result of Theorem~\ref{thm:mixup_privacy}, we investigate the data augmentations previously considered in Section~\ref{sec:experiments} with additional Laplacian noise, also in the setting of a gradient matching attack. Figure \ref{fig:variations_mixing} shows that the benefits of Laplacian noise which we only prove for mixup also extend empirically to variants of mixing data augmentations such as CutMix and MaxUp. In particular, combining MaxUp with Laplacian noise of sufficient strength ($s=16/255$) completely shuts down the data poisoning attack via adaptive gradient matching, significantly improving upon numbers reached by MaxUp alone. 


\section{Discussion}
\label{sec:discussion}
Strong data augmentations have previously been used to improve generalization in neural networks. In this work, we first show that such augmentations also yield a state-of-the-art empirical defense against a range of data poisoning and backdoor attacks.  We then go a step further and analyse these data augmentations theoretically through the lens of differential privacy, due to its connections to poisoning robustness. We prove that mixup augmentation enhances the defensive guarantees obtained by adding noise to inputs, improving standard guarantees at least linearly in mixture width. Finally, we apply these findings practically, evaluating the effects of mixup data augmentations combined with Laplacian input noise.

\section*{Acknowledgements}
This work was supported by the JP Morgan Faculty Research Awards Program, the DARPA GARD program, and the DARPA YFA program.  Additional support was provided by DARPA QED and the National Science Foundation DMS program.

\bibliography{references}
\bibliographystyle{icml2021}

\appendix
\section{Appendix}
Experimental details for the experiments shown in the main paper are contained in this document.

\subsection{Backdoor Attacks}
For the patch attack, we insert patches of size $4 \times 4$ into CIFAR train images from target class and test images from victim class. The patches are generated using a Bernoulli distribution and are normalized using the mean and standard deviation of CIFAR training data. The patch location for each image is chosen at random. To evaluate the effectiveness of the backdoor attack and our proposed defenses, we train a ResNet-18 model on poisoned data with cross-entropy loss. The model is trained for 80 epochs using SGD optimizer with a momentum of 0.9, a weight decay of 5e-4 and learning rate of 0.1 which we reduce by a factor of 10 at epochs 30, 50 and 70. A batch size of 128 is used during training.  

We run our experiments for HTBD and CLBD in Table 6 by implementing mixup and CutMix in the publically available framework of \citet{schwarzschild2020just}, and using this re-implementation for our comparison with the hyperparameters proposed there.

\subsection{Targeted Data Poisoning}
We run our experiments for feature collision attacks in Table 4 by likewise using the framework of \citet{schwarzschild2020just}, running the defense with the same settings as proposed there and following the constraints considered in this benchmark. For gradient matching we likewise implement a number of data augmentations as well as input noise into the framework of \citet{geiping2020witches}. We run all gradient matching attacks within their proposed constraints, using a subset of 1\% of the training data to be poisoned for gradient matching and an $\ell^\infty$ bound of 16/255. For all experiments concerning gradient matching we thus consider the same setup of a ResNet-18 trained on normalized CIFAR-10 with horizontal flips and random crops of size 4, trained by Nesterov SGD with $0.9$ momentum and 5e-4 weight decay for 40 epochs for a batch size of 128. We drop the initial learning rate of 0.1 at epochs 14, 24 and 35 by a factor of 10. For the ImageNet experiments we consider the same hyperparameters for an ImageNet-sized ResNet-18, albeit for a smaller budget of $0.01\%$ as in the original work.

Comparing to poison detection algorithms, we re-implement \textit{spectral signatures} \citep{tran2018spectral}, \textit{deep K-NN} \citep{peri2019deep} and \textit{Activation Clustering} \citep{acchen2018detecting} with hyperparameters as proposed in their original implementations. For differentially private SGD, we implement Gaussian gradient noise and gradient clipping to a factor of 1 on the mini-batch level (otherwise the ResNet-18 architecture we consider would be inapplicable due to batch normalizations), and vary the amount of gradient noise with values ($0.0001$, $0.001$, $0.01$) to produce the curve in Fig.~2.

To implement data augmentation defenses we generally these data augmentations straightforward as proposed in their original implementations, also keeping components such as the late start of Maxup after 5 epochs described in \citet{gong2020maxup} and the randomized activation of CutMix described in \citet{zhang2017mixup}.

\paragraph{Mixup Experiments with noise}
We repeat the same setup for the empirical experiments in Fig.~4~b), increasing the mixing strength of mixup for several noise levels. We extend this to other data augmentation in Fig.~5 and plot the trade-off between security and performce there.

\end{document}